\def\year{2019}\relax
\documentclass[letterpaper]{article}
\usepackage{aaai19}
\usepackage{times}
\usepackage{helvet}
\usepackage{courier}
\usepackage{url}  
\usepackage{graphicx}
\usepackage[table,xcdraw]{xcolor}
\usepackage{algorithm}
\usepackage{algorithmic}

\usepackage{amssymb}
\usepackage{amsmath}
\usepackage{amsthm}
\usepackage[toc,title]{appendix}
\usepackage{bm}
\usepackage{color}
\usepackage{fancyhdr}
\usepackage[acronym,shortcuts,nonumberlist]{glossaries}

\usepackage{subfigure}

\usepackage{multirow}
\usepackage{comment}

\newtheorem{theorem}{Theorem}
\newtheorem{lemma}{Lemma}
\newtheorem{corollary}{Corollary}
\newtheorem{remark}{Remark}

\newcommand{\ignore}[1]{}
\newcommand{\ph}[1]{\vspace{0mm} \noindent \textbf{#1.}} 

\newcommand{\tempcut}[1]{}

\definecolor{mypink3}{cmyk}{0, 0.7808, 0.4429, 0.1412}

\newcommand{\MEtime}{{$O\left( \frac{n\sqrt{N}}{\epsilon}\sqrt{\log\left(\frac{1}{\delta}\right) } \right)$}}
\newcommand{\ME}{{\sc MedianElimination}}
\renewcommand{\ME}{{\sc BoundedME}}
\newcommand{\GM}{{\sc Greedy-MIPS}}
\newcommand{\LSH}{{\sc LSH-MIPS}}
\newcommand{\PCA}{{\sc PCA-MIPS}}
\newcommand{\DM}{{\sc Diamond-MSIPS}}
\newcommand{\SM}{{\sc Sample-MIPS}}
\newcommand{\newBandit}{{Multi-Armed Bandit with Bounded Pulls}}
\newcommand{\newBanditShort}{{MAB-BP}}
\newcommand{\RPT}{{\sc RPT-MIPS}}

\begin{document}
\title{A Bandit Approach  to Maximum Inner Product Search}

\author{
  Rui Liu\\
  Computer Science and Engineering\\
   University of Michigan, Ann Arbor\\
  \texttt{ruixliu@umich.edu} \\
\And
  Tianyi Wu\\
 Computer Science and Engineering\\
   University of Michigan, Ann Arbor\\
  \texttt{tianyiwu@umich.edu} \\
\And
  Barzan Mozafari\\
  Computer Science and Engineering\\
   University of Michigan, Ann Arbor\\
  \texttt{mozafari@umich.edu} \\
}

\date{}
\maketitle
\pagenumbering{gobble}

\begin{abstract}
 There has been substantial research on sub-linear time approximate algorithms for 
 	 \emph{Maximum Inner Product Search} (MIPS). 
   To achieve fast query time, state-of-the-art techniques require significant preprocessing, 
  which can be a burden when the number of subsequent queries is not sufficiently large to amortize the cost.
  Furthermore,  existing methods do not have the ability to \emph{directly} control the suboptimality of their approximate results
     with theoretical guarantees. 
In this paper, we propose the first approximate algorithm for MIPS that does not require any preprocessing, 
  and allows users to control and bound the suboptimality of the results.
We cast MIPS as a Best Arm Identification problem, and introduce a new bandit setting 
	that can fully exploit the   special structure of MIPS. 
Our approach outperforms  state-of-the-art methods  on both synthetic and real-world datasets.        
  \end{abstract}


\section{Introduction}
The problem of \emph{Maximum Inner Product Search} (MIPS) has received significant attention in recent years \cite{yu2017greedy,nips2014-best,neyshabur2015symmetric} as a key step in many machine learning
algorithms and applications.
For instance, it appears in matrix-factorization-based recommender  systems
\cite{koren2009matrix,cremonesi2010performance,liu2015robust}, multi-class prediction \cite{dean2013fast,jain2009active}, structural SVM 
\cite{joachims2006training,joachims2009cutting}, and vision applications \cite{dean2013fast}. 
The MIPS problem can be formally defined as follows: given a collection of $n$ data vectors, $\mathcal{S} = \{v_1,v_2,\cdots, v_n\}$, 
where $v_i\in \mathbb{R}^{N}, 1\leq i\leq n$, and a query vector $q\in \mathbb{R}^{N}$,
	the goal is to find $v^*\in \mathcal{S}$ that maximizes (or approximately maximizes) the inner product $q^Tv^*$. 
In other words, MIPS is the following problem:
\begin{equation}
v^* =\mathop{arg} \mathop{max}_{v\in \mathcal{S}} q^Tv
\end{equation}
The na\"ive linear search for solving MIPS  requires $O(n\cdot N)$ time to exhaustively compute all $n$ inner products, which  can be daunting for 
	 massive datasets (large $n$) and/or high-dimensional data (large $N$). 
This has led to significant interest in 
	devising sub-linear time algorithms to solve the MIPS problem approximately,
		including the best paper award at NIPS'14 \cite{nips2014-best} and 
			many other proposals \cite{yu2017greedy,neyshabur2015symmetric,bachrach2014speeding}.

\ignore{
For example, previous work has proposed greedy strategies \cite{yu2017greedy}, 
	LSH-based techniques---such as \cite{neyshabur2015symmetric} and the best paper award at NIPS'14 \cite{nips2014-best}---
		and PCA-based approaches \cite{bachrach2014speeding}.}

\ph{Motivation I: Preprocessing Overhead} 
Despite their different merits, existing approximate techniques for MIPS all share a common pattern: 
	they require a substantial time for preprocessing the data set $\mathcal{S}$, during which they construct a
		data structure that can then be used to answer queries more efficiently. 
For example, they construct a hash table \cite{nips2014-best}, a sorted index \cite{yu2017greedy}, 
	 space partition trees \cite{bachrach2014speeding}, or other data structures \cite{auvolat2015clustering}. 
	 Existing methods require various data structures to be built during preprocessing time. 
In fact, the preprocessing time is, in some cases, so large that it even exceeds the time complexity of the na\"ive approach for answering $\log n$ queries, i.e., $O(N n \log n)$ \cite{yu2017greedy}. 
We have summarized the preprocessing time of the previous techniques in Table \ref{TB:pre_time}.
The rationale is that query times will be faster after the preprocessing step. 
The preprocessing time is therefore justified when 
		there are many queries and the set $\mathcal{S}$ remains the same, i.e.,
			once the preprocessing is done it would benefit many subsequent queries.
However, there are many cases where either the number of queries is relatively small (or even $1$), 
	or the set $\mathcal{S}$ changes frequently,
	e.g., the approximate Linear Minimization Oracle (LMO) in Matching Pursuit and 
	 Frank-Wolfe optimization \cite{locatello2017unified,jaggi2013revisiting}.
In these scenarios, the preprocessing time is simply a burden. 
In this paper, our first motivation is to design an approach to MIPS that will not require any preprocessing, while still achieving 
	a query speed-up better than those approaches that do require it. 

\ph{Motivation II: Suboptimality Bounds}	
	 As summarized in Table~\ref{TB:pre_time}, specific parameters such as the number of hash functions or depth of partition trees are used to trade search accuracy with search efficiency.
	 Various data structures with pre-specified parameters~\cite{nips2014-best,bachrach2014speeding,auvolat2015clustering} are built before any query is given, which means that the trade-off is somewhat fixed for all queries. 
	 Thus, the computational cost for a given query is fixed. However, in many real-world scenarios, each query might have a different computational budget. 
	 \cite{yu2017greedy} proposed a greedy method where the user could control the computational budget for each query. The computational budget can be viewed as an efficiency-accuracy knob. Nevertheless, the user can't get a solution with a guaranteed level of optimality in general.\footnote{A more detailed treatment of the previous work is deferred to Section of Related Work} It is crucial for practitioners to have a knob with which the user can explicitly request and guarantee a certain level of optimality for each query~\cite{approx_chapter,mozafari_eurosys2013,mozafari_sigmod2014_diagnosis}, e.g., to know how accurate the solution would be if more computational budget were allowed~\cite{blinkml-tr}. 
Thus, our second motivation in this paper is to design a MIPS algorithm that can bound (and directly control) the suboptimality of 
the returned answer, regardless of the data distribution. 
	Specifically, for any $0< \epsilon< 1$ and $0< \delta < 1$ chosen by the user, 
		the algorithm must be able to guarantee that, 
			with probability at least $1-\delta$, the returned solution $\hat{v}$ is $\epsilon$-optimal with respect to optimal solution $v^*$,
			i.e., $\frac{1}{N}q^T v^* - \frac{1}{N}q^T\hat{v}<\epsilon$. 
This would help us to better understand the behavior of our algorithm theoretically, and would offer 
	a flexible knob for trading off error and computational efficiency in practice.

\begin{table*}[t]
\centering
\caption{State of the art algorithms for MIPS. Here, $n$ is the number of data vectors and $N$ is the vectors' dimensionality.}
\vspace{0.2cm}
\scalebox{0.85}{
\begin{tabular}{|p{2.3cm}|c|c|p{5cm}|p{4cm}|}
\hline
Method & Preprocessing & Query  &  Theoretical Guarantees & Notes\\
             & Time & Time & & \\
\hline
\ME{}  (our method) 
& $0$ & \MEtime{} & Guaranteed to return an $\epsilon$-optimal solution with probability at least $1-\delta$  & User can choose any desired error $0$$<$$\epsilon$$<$$1$ and confidence 
$0$$<$$\delta$$<$$1$  \\
\hline
\GM{}~\cite{yu2017greedy} & $O(Nn\log n)$  & $O(BN)$  & 
                  No guarantees in general.  
                  (They   guarantee optimality with high probability, only for uniformly distributed data and budget $B \geq O(N\log(n)n^{\frac{1}{N}})$
		 & 	For non-uniform data, the results can be arbitrarily poor (e.g., 
				if the largest coordinate of $q^Tv$ is identical for all $v\in \mathcal{S}$, 
					the output will be a random subset)
\\
\hline
 \LSH{}~\cite{nips2014-best,neyshabur2015symmetric} & $O(N n a b)$ & $O(\frac{nN}{2^a}b)$ & 
  They guarantee to return the optimal vector $v^*$ for query $q$ 
 		 with probability  \[1 - \left( 1- \left(1-\frac{\cos^{-1}(q^Tv^*)}{\pi} \right)^a\right)^b\] where 
		 $a$ is the number of bits in each hyper LSH function and
 $b$ is the number of hyper LSH functions 
		 &  
		Since $v^*$ is unknown \emph{a priori}, the users cannot control the lower bound of this probability 
		(e.g., if $q^Tv^* = -1$, this probability will always be $0$ regardless of the values chosen for $a$ and $b$)
\\
\hline
  \RPT{}~\cite{keivani2017improved} & $O(LN n \log n)$ & $O(L \log n)$ & 
  They guarantee to return the optimal vector $v^*$ for query $q$ 
 		 with probability upper bounded by some potential function depending on $q$, $S$ and $L$, where $L$ is the number of trees 
		 &  
		Since $q$ is unknown \emph{a priori}, the users cannot control the lower bound of this probability 
\\
\hline 
\PCA{}~\cite{bachrach2014speeding} & $O(N^2n)$& $O(\frac{nN}{2^d})$ & None  &  $d$ is the depth of the PCA tree \\
\hline 
\end{tabular}
}
\label{TB:pre_time}
\end{table*}

\ph{Our Approach}
Our approach to MIPS  is inspired by the Multi-Armed Bandit (MAB) problem \cite{bubeck2009pure,audibert2010best,even2006action,jamieson2014lil}.
MAB is a predominant model for characterizing the tradeoff between exploration and exploitation in decision-making settings.
	In MAB, there are $n$ arms; each time we pull an arm, it returns a reward
(e.g. a reward generated by sampling from a Guassian distribution).
The true mean of an arm is defined as the mean of the distribution from which its rewards are sampled.
  The goal in MAB is to either (1)   accumulate as much reward as possible, or (2)   identify the best arm (i.e., the one with the highest true mean).
In our paper, our goal is the latter. We cast MIPS as a Best Arm Identification problem:
	we can treat each data vector as an arm, where pulling it means multiplying one of its coordinates with the corresponding 		coordinate from the query vector. 
We must then dynamically decide how many more floating-point multiplications
	 to perform for each inner product, based on the partial results of all  inner products thus far.

There are two different stopping conditions for the 
Best Arm Identification problem: fixed confidence and fixed budget. 
With fixed confidence, the MAB algorithm 
	seeks to minimize the sample complexity---the number of pulls used---while guaranteeing 
		that the returned arm is $\epsilon$-optimal
			 with probability at least $1-\delta$  for any given $0< \delta<1, 0< \epsilon <1$. 
In the fixed budget setting, the MAB algorithm stops once it has used its budget in terms of the sample complexity,  
	 while seeking to return an arm whose true mean is as close as possible to that of the best arm. 
Although our \textbf{Motivation II} is inline with the fixed confidence setting, 		 
	the existing MAB methods for fixed confidence are not effective for the MIPS problem. 
 This is because the existing algorithms are designed for i.i.d. rewards drawn from some unknown distribution over an \emph{infinite} population \cite{bubeck2009pure,audibert2010best}, and thus require many pulls to achieve an accurate estimate of an arm's true mean. 
In MIPS, however,  the useful number of pulls for any arm is upper bounded by $N$ (the vectors's dimension).
In other words,  rewards are sampled \emph{without replacement} from a discrete uniform distribution over a \emph{finite} population. 	
	Therefore, by exploiting this structure, we should be able to significantly lower the number of pulls.
In this paper, we introduce a new setting for Best Arm Identification problem with fixed confidence that suits the special structure of MIPS.
We also propose an algorithm, called \ME{}, inspired by the
	Median Elimination framework \cite{even2002pac}, using a tight statistical bound for sampling without replacement.   

In summary, we make the following contributions:
\begin{itemize}
\item We identify two  desirable motivations for the MIPS problem  that are important in practice but overlooked 
by existing solutions (\textbf{Movitaion I} and \textbf{Motivation II}). We introduce a new MAB setting for Best Arm Identification, where the rewards for each arm are sampled from a large but finite list. We call this  setting Multi-Arm Bandit with Bounded Pulls (MAB-BP). 

\item We propose a new algorithm for MAB-BP, called 
	\ME, which extends Median Elimination with a tight statistical bound.
When applied to MIPS, \ME{}  enjoys  a significantly lower sample complexity than all previous MAB methods developed for  fixed confidence setting. 
  More importantly, as a bandit approach, 
		\ME{} does not require any preprocessing (unlike previous MIPS solutions).

\item Our extensive experiments on both synthetic and real-world datasets  show that \ME's query time is $5$--$10\times$ faster 
			than state-of-the-art MIPS algorithms, despite their use of preprocessing. 

\end{itemize}


\section{Related Work}
\label{sec:related}


\subsection{Existing Approaches to MIPS}

  There are a number of sampling-based methods for MIPS.
For example, \SM{}~\cite{cohen1999approximating} is a scheme that samples $(i, j)\in$$\{1,\cdots,n\}$$\times\{1,\cdots,N\}$ with probability proportional to $v_{i}^{(j)}q^{(j)}$. However, it requires that all candidate and query vectors be nonnegative.  
  \DM{} \cite{ballard2015diamond}
	is another sampling-based approach
		which solves a similar problem, called 
   \emph{maximum squared inner product search (MSIPS)}.
 The goal in MSIPS is to find candidate vectors $v\in \mathcal{S}$ for which $(q^Tv)^2$ is maximized.  
 The solution to MSIPS, however, can be very different than that of MIPS, e.g., the former might return a $v$ whose inner 
 	product with $q$ is a large negative value.
	
Another popular approach   is to reduce MIPS to the nearest neighbor search problem,
	which can then be solved using locality-sensitive hashing \cite{nips2014-best,neyshabur2015symmetric}, neighbor-sensitive hashing~\cite{mozafari_pvldb2015_ksh},
PCA-trees \cite{bachrach2014speeding}, or K-Means approaches \cite{auvolat2015clustering}. 
Nonetheless, all of these methods share a common pattern. Before answering any queries, they conduct
	 a preprocessing  on $\mathcal{S}$
to construct an approach-specific data structure, e.g., hash table in \LSH{}, space partition trees in \PCA{}, or cluster centroids in \cite{auvolat2015clustering}. 
These data structures only contain information about the vectors in $\mathcal{S}$ and are independent of any query. 
Then, for each query, they use 
	an efficient procedure 
	 on their preconstructed data structure
		to select a set of candidate vectors (i.e., a subset of $\mathcal{S}$). 
They perform an exact ranking on this candidate set to return the best vector.  
A recent approach, Randomized Partitioning Tree (\RPT{})~\cite{keivani2017improved},
	 builds a partitioning tree on top of an LSH scheme to solve the MIPS problem. 
\RPT{}  guarantees the exact solution with a probability that depends on the vector set $S$ and
	the given query.  
However, \RPT{} cannot  directly control  	
	   the quality of the returned vector (i.e., suboptimality bounds in \textbf{Motivation II}).
Another approach is \GM{} \cite{yu2017greedy}, which builds on the same algorithmic pattern, 
	but also provides a budget for the size of the candidate set. 
This parameter is the only mechanism that implicitly controls the tradeoff between  precision and  query time.   
Table~\ref{TB:pre_time} summarizes     the 
    theoretical guarantees offered by some of the recent MIPS algorithms.

\subsection{Existing Approaches to MAB}
Best Arm Identification is a popular setting in MAB that aims to identify the best arm by dynamically deciding on how many times to pull each arm.
In the fixed budget setting, 
	the main idea behind algorithms such as Successive Halving \cite{karnin2013almost,jamieson2016non} and Successive Rejects \cite{audibert2010best} is to dynamically allocate the budget to the different arms in order to remove the bad arms in a round-by-round fashion until only one arm is left. 
Naturally, these methods tend to use up their budget 
 	in order to return an arm whose true mean is as close as possible to the optimal arm.  
The fixed budget setting isn't suitable for our problem in this paper.
In the fixed confidence setting,  algorithms share the same idea of 
	dynamically pulling arms and removing the unpromising ones from consideration, 
 	such as Successive Elimination \cite{even2006action}, Exponential Gap Elimination \cite{karnin2013almost}, LUCB \cite{kalyanakrishnan2012pac}, and Lil'UCB \cite{jamieson2014lil}.   
	These algorithms are inline with our \textbf{Motivation II}, as they also seek to minimize sample complexity while guaranteeing that the returned arm is within a pre-specified proximity of the optimal arm.
	However, they cannot be directly applied to the MIPS problem because they assume the rewards are i.i.d. samples from some unknown distribution over an \emph{infinite population},
	 whereas in MIPS, the rewards are sampled without replacement from a discrete uniform distribution over a \emph{finite} list. 

\section{A Bandit Approach to MIPS}
\label{sec:mips}

In this section, we show how the MIPS problem can be viewed as a Best Arm Identification problem with fixed confidence.
 As previously mentioned, the goal in MIPS 
	is to solve the problem:
\[
v^* =\mathop{arg} \mathop{max}_{v\in \mathcal{S}} q^Tv
\]
where $\mathcal{S} = \{v_i | v_i\in \mathbb{R}^{N}, 1\leq i\leq n\}$ is a collection of $n$ data vectors 
 and $q\in \mathbb{R}^{N}$ is a given query.

We cast MIPS as a bandit problem as follows.
	For every data vector $v_i\in \mathcal{S}$, we consider a corresponding arm 
	$a_i$, whose reward has the following true mean 
	$p_i=\frac{1}{N}\sum_{j=1}^{N} v_{i}^{(j)}q^{(j)} = \frac{v_i^Tq}{N}$,  
		where $v_{i}^{(j)}$  and $q^{(j)}$ are the $j$-th coordinates of $v_i$ and $q$, respectively, for $1\leq j\leq N$. 	
	When the arm $a_i$ is pulled $t$ times ($1\leq t\leq N$), 
		its rewards are generated by taking $t$ i.i.d. samples \emph{with} replacement
		 from its reward list $R_i$,  
			defined as the set $R_i=\{ v_{i}^{(1)}q^{(1)}, v_{i}^{(2)}q^{(2)}, \cdots, v_{i}^{(N)}q^{(N)} \}$.

Each time the arm $a_i$ is pulled, returning a reward corresponds to a floating-point operation in MIPS 
	for multiplying one of the coordinates from $v_i$ with its counterpart from $q$, i.e., computing $v_{i}^{(j)}q^{(j)}$ for some $1$$\leq$$j$$\leq$$N$.
The reward lists are initially unknown, but the more we pull an arm, the more we learn about its reward list. 
Our goal in MIPS is to find the arm with the highest true mean---the vector whose inner product with $q$ is the 
highest---using as few floating point operations as possible. 

Unfortunately, in a traditional bandit setting,
	even if we pull the $a_i$ arm $N$ times, we still do not know the exact true mean of $R_i$, i.e., the  
		exact inner product of $v_i$ and $q$. 
This is because traditional bandit problems are designed for 
	an unknown distribution over an infinite population, and hence rely on sampling \emph{with} replacement. 
However, to solve MIPS,
		the rewards are drawn from a finite reward list---i.e., $N$ coordinates.
Thus, if we can exploit this structure and use 
	sampling \emph{without} replacement,
		we must be able to pull the arms significantly fewer times than in a traditional bandit.
Further, once we have pulled the $a_i$ arm $N$ times, we should know the entire content of $R_i$,  
equivalent to the exact computation of the inner product between $v_i$ and $q$.
Next, we formally define this new bandit setting.

\section{Multi-Armed Bandit with Bounded Pulls (MAB-BP)}
\label{sec:bandit}

We now formally introduce a new Multi-Armed Bandit setting, which we call \emph{Multi-Armed Bandit with Bounded Pulls (MAB-BP)}. 
Assume a set of $n$ arms $A = \{a_1, \cdots, a_n\}$.
Each arm $a_i$ is associated with a reward list $R_i = \{R_i^{(1)}, R_i^{(2)},\cdots, R_i^{(N)}  \}$, where $N$ is the size of the reward list.
Here, we assume $R_i^{(j)} \in [0, 1]$, but a similar analysis applies as long as the reward value is bounded.
Every time an arm $a_i$ is pulled, a reward is returned by sampling a value \emph{without} replacement from its reward list $R_i$.
Denote the true mean of reward for arm $a_i$ as $p_i = \frac{1}{N}\sum_{j = 1}^N R_i^{(j)} $.
Thus, once an arm is pulled $N$ times, the mean of the returned rewards is exactly equal to the true mean $p_i$.
Our goal is the same as in a traditional Best Arm Identification: to identify an $\epsilon$-optimal arm
	 with probability at least $1-\delta$ using as few pulls as possible, 
	 where $\epsilon$ and $\delta$ are provided by the user.
We say that an arm $\hat{a}$ is an $\epsilon$-optimal arm if $p_{a^*} - p_{\hat{a}} < \epsilon$, where $a^*$ is the 
optimal arm.

It is easy to see that by choosing $R_i^{(j)}=v_{i}^{(j)}q^{(j)}$ for $1\leq j\leq N$, one can cast MIPS 
	as a  MAB-BP problem.  However, note that MAB-BP can be used 
		to solve any problem of the form:
\[
\mathop{arg} \mathop{max}_{1\leq i\leq n} \sum_{j=1}^N f(i, j)
\]
where $f$ can be an arbitrary function.

For MIPS, $f(i, j) = v_i^{(j)}q^{(j)}$. 
However, one can also use
MAB-BP to solve the Nearest Neighbor Search (NNS) problem: given a collection of vectors $\mathcal{S} =     \{v_1,\cdots, v_n\}$,
 where $v_i\in \mathbb{R}^{N}, 1\leq i\leq n$, and a query vector $q\in \mathbb{R}^{N}$,
 the goal is to find $v^*$$\in$$\mathcal{S}$ that is closest to $q$, 
 i.e.,  $ \|q-v^*\|^2 = \sum_{j=1}^N (q^{(j)} - v^{*(j)})^2$ is minimized.
 In this case, $f(i, j)$$=$$- (q^{(j)} - v_i^{(j)})^2$.

\section{\ME{}: An Algorithm for Solving \newBanditShort{}}

Existing bandit algorithms   are sub-optimal for \newBanditShort{}. 
The fundamental reason is   bandit algorithms have to estimate 
the minimum number of samples needed to obtain an  estimate $\hat{p}$ of the true mean $p$
that satisfies the given error $\epsilon$ and confidence $\delta$ requirements, i.e., 
$\mathbb{P}\left[ \hat{p} - p \leq \epsilon  \right] \geq 1 - \delta$.

The efficiency of a bandit algorithm depends on how accurately it can estimate the number of required samples,
	based on the   reward values it has observed for each arm.
This goal is achieved using  concentration inequalities~\cite{boucheron2013concentration}.
 Since in traditional bandit the rewards are typically assumed to be sampled from a sub-Gaussian distribution  over an infinite population~\cite{jamieson2014best}, 
 	these algorithms often rely on 
	Hoeffding's bound or the law of iterated logarithm (LIL) bound to determine the sample size for mean estimation
		of the reward distribution 
		 over an infinite population. 
 However, as noted earlier, the reward values in MAB-BP  are sampled without replacement and from a finite list.
Thus, an algorithm that can exploit this additional information should be able to solve the MAB-BP problem
	more efficiently (i.e., with lower sample complexity).

Next, we derive a concentration inequality for sampling without replacement 
	and then present our algorithm.

\subsection{A Concentration Inequality for MAB-BP}
\label{sec:concentration}

We use the following corollary from \cite{bardenet2015concentration}. 

\begin{corollary}[Corollary 2.5 in \cite{bardenet2015concentration}] \label{CO:sample_without_replacement_original}
Let $\mathcal{X} = (x_1, x_2, \cdots, x_N)$ be a finite set of size $N>1$ in $[a, b]$ with mean $\mu = \frac{1}{N}\sum_{i=1}^Nx_i$, and $(X_1, \cdots, X_m)$ be a list of size $m< N$ sampled without replacement from $\mathcal{X}$.
  Then for any $m\leq N$, and any $\delta \in [0,1]$, it holds 
\begin{equation} \small
\mathbb{P}\left[ \frac{1}{m}\sum_{t=1}^m X_t - \mu \leq (b-a)\sqrt{\frac{\rho_m \log(1/\delta)}{2m}} \right] \geq 1-\delta
\end{equation}
where $\rho_m$ is defined as
\begin{equation} \small
\rho_m =\mathop{min} \left\{
(1-\frac{m-1}{N}),
(1-\frac{m}{N})(1+1/m)
\right\}
\end{equation}
\end{corollary}

Based on this corollary, we could get the following concentration inequality for sampling without replacement, as in the following lemma\footnote{All omitted proofs can be found in the supplementary material.}.  
\begin{lemma} \label{LM:sample_without_replacement}
Let $\mathcal{X} = (x_1, x_2, \cdots, x_N)$ be a finite set of size $N>1$ in $[a, b]$ with mean $\mu = \frac{1}{N}\sum_{i=1}^Nx_i$, and $(X_1, \cdots, X_m)$ be a list of size $m< N$ sampled without replacement from $\mathcal{X}$.
Then, for any given  $0<\epsilon<1, 0<\delta<1$,  if 
\begin{equation}
m = \mathop{min}\{\frac{u+1}{1+\frac{u}{N}}, \frac{u+\frac{u}{N}}{1+\frac{u}{N}} \}
\end{equation}
where $u= \frac{\log(1/\delta)}{2} \frac{(b-a)^2}{\epsilon^2}  $,
then we have 
\begin{equation}
\mathbb{P}\left[ \frac{1}{m}\sum_{t=1}^m X_t - \mu \leq  \epsilon  \right] \geq 1-\delta
\end{equation}.
\end{lemma}

We can see that as the error bound $\epsilon$ approaches $0$, the required sample size $m$ will approach the finite set size $N$, but never exceed $N$. 
This matches our previous intuition that it is ineffective to pull one arm more than $N$ times in \newBanditShort{}. 
It is worth noting that a lemma similar to Lemma \ref{LM:sample_without_replacement} can   be derived  
	to show  that $\mathbb{P}\left[ \frac{1}{m}\sum_{t=1}^m X_t - \mu \geq - \epsilon  \right] \geq 1-\delta$,
based on a   corollary similar to Corollary \ref{CO:sample_without_replacement_original}.

\subsection{The \ME{} Algorithm}
\label{sec:ourbandit}
 Our proposed algorithm, \ME{}, is based on the median elimination strategy, but tailored to our MAB-BP setting. 
The basic idea of median elimination strategy is that, given a set of arms $A$, we pull each of these arms for a certain number of times to update their empirical means, discard the worst half in terms of their updated empirical means thus far, and repeat until only one arm remains.
In Algorithm~\ref{alg:me}, 
	we present \ME{} for the more general case of 
		identifying the top $K$ arms with the highest true means of rewards.
(The best arm identification is a special case, where $K$=$1$.) 
To be more specific about the difference between Algorithm~\ref{alg:me} and the general median elimination strategy, the number of times that we pull each remaining arm is $t_l - t_{l-1}$ for the $l$-th iteration, where $t_l$ is defined in the line $7$ of Algorithm~\ref{alg:me}. 
Since we want to identify the top $K$ arms, at the end of the $l$-th iteration, we discard $\left \lceil\frac{|S_l| -K}{2} \right\rceil$ arms with least empirical means thus far where $|S_l|$ is the number of remaining arms at the begining of $l$-th iteration, rather than discarding the worst half. 
In addition, we will stop when only $K$ arms remain.

 To simplify our notation, we enumerate the arms according to their true mean, i.e., $p_1$$>$$p_2$$>$$\cdots$$>$$p_n$.
Let $T^*=\{ 1, 2, \cdots, K \}$ be the set of best $K$ arms.\footnote{We use the index $i$ instead of $a_i$
	to simplify the notation.} 
For any set $T$ consisting of $K$ arms, we say that $T$ is \textbf{$\epsilon$-optimal} if
$\tilde{p}_{T^*} - \tilde{p}_T \leq \epsilon$,
where $\tilde{p}_S$ is the $K$-th highest true mean among the arms in $S$. 
We also define the \textbf{suboptimality} of $T$ to be $\tilde{p}_{T^*} - \tilde{p}_T$. 
 
 Given  $\epsilon$ and $\delta$ provided by the user,
 	  \ME's goal is to identify a set of $K$ arms 
	  that is  $\epsilon$-optimal with probability at least $1-\delta$, using as few pulls as possible.

 \begin{algorithm}
\caption{\ME{} Algorithm (for top-$K$)}
\label{alg:me}
\begin{algorithmic}[1]
\STATE{\textbf{input:} $K\geq 1, \epsilon>0, \delta>0$, and a set of arms $A$}
\STATE{\textbf{output:} a set of $K$ arms that is $\epsilon$-optimal with probability $1-\delta$    }
\STATE{}
\STATE set $S_1 = A, \epsilon_1 = \frac{\epsilon}{4}, \delta_1 =\frac{\delta}{2}, l =1$
\STATE set $t_0=0$

\WHILE{$|S_l| > K$}
\STATE{set $t_l = m\left( \frac{2}{\epsilon_l^2} \log\left( \frac{2(|S_l|-K)}{\delta_l \left(\left \lfloor\frac{|S_l| -K}{2} \right\rfloor +1  \right)} \right)    \right)$}
\STATE{ Pull every arm $a\in S_l$ for $t_l - t_{l-1}$  times, and let $\hat{p}_a^l$ denote its empirical mean since the begining of the algorithm}
\STATE{Find the $\left \lceil\frac{|S_l| -K}{2} \right\rceil$-th value of $\hat{p}_a^l$ in ascending order, and denote it as $\bar{p}_l$}
\STATE{$S_{l+1} = S_l\setminus \{ a\in S_l: \hat{p}_a^l \leq \bar{p}_l  \} $ (more precisely, remove $\left \lceil\frac{|S_l| -K}{2} \right\rceil$ arms with
	the least empirical means thus far)}
\STATE{$\epsilon_{l+1} = \frac{3}{4}\epsilon_l,  
\delta_{l+1} = \frac{\delta_{l}}{2}, l= l+1$}
\ENDWHILE
\STATE {\bfseries return} $S_l$
\end{algorithmic}
\end{algorithm}

We use the following function to simplify our presentation:
\begin{equation} \label{EQ:m_func} \small
m(u) = \mathop{min}\left\{\frac{u+1}{1+\frac{u}{N}}, \frac{u+\frac{u}{N}}{1+\frac{u}{N}} \right\}
\end{equation}
We can see that $m(u) < N$ as long as $u > 0$.
We have the following lemma.

\begin{lemma} \label{LM:one_round_pac}
For   Algorithm \ref{alg:me}, at any iteration $l$,  we have
$\mathbb{P}\left[ \tilde{p}_{S_l}\leq \tilde{p}_{S_{l+1}} +\epsilon_l  \right] \geq 1-\delta_l$
.
\end{lemma}

Based on the above Lemma~\ref{LM:one_round_pac}, we could get the main theoretical property of Algorithm~\ref{alg:me}.   

\begin{theorem} \label{THM:pac_me}
The \ME{} algorithm (Algorithm~\ref{alg:me}) is guaranteed to return $\epsilon$-optimal solution with probability at least $1-\delta$.
\end{theorem}

Note that \ME{} is never slower than the na\"ive search, which has the $O(nN)$ time complexity:
\begin{corollary}
For each arm, the number of times it is pulled by Algorithm~\ref{alg:me} is upper-bounded by $N$.
\end{corollary}

\ME{}'s time complexity is also lower than   Median Elimination~\cite{even2002pac}, which is $O\left(\frac{n}{\epsilon^2} \log(\frac{1}{\delta}) \right)$:
 \begin{corollary} \label{CO:time_complexity}
The time complexity of Algorithm~\ref{alg:me} is \MEtime{}.
\end{corollary}

\begin{remark}
When Algorithm~\ref{alg:me} is applied to the MIPS problem, the above bound indicates that the running time is sublinear in the dimension of vectors, but linear in the size of vector set $\mathcal{S}$. 
This implies that our approach is especially effective for very high-dimensional data. 
To mitigate the potential issue of linear dependence on the size of $\mathcal{S}$, we could exploit the geometric structure or similarity among vectors from $\mathcal{S}$. 
The tradeoff here is that this would now require some preprocessing. For example, we could find the convex hull of the set $\mathcal{S}$ first, and then only focus on the set of extreme points that form the convex hull, because the solution of the MIPS problem is guaranteed to always include at least one of these extreme points. 
Thus, when the number of extreme points is much smaller than the size of $\mathcal{S}$, our algorithm becomes sublinear in the size of $\mathcal{S}$.
\end{remark}

\section{Experiments}

Our experiments aim to (1) empirically validate the theoretical guarantees of Theorem~\ref{THM:pac_me} and 
 (2) compare  our method against the state-of-the-art.

\ph{Datasets}
Since Theorem~\ref{THM:pac_me} is a worst-case guarantee, we use an adversarily-generated synthetic dataset to verify its correctness.
Then, 
we use both synthetic and real-world datasets to compare 
	our algorithm with several state-of-the-art techniques.
For each dataset, we used $10^4$ vectors with $10^5$ dimensions.

\ph{Baselines}
We compared our method against the following state-of-the-art methods:
\begin{itemize}
\item \LSH{}~\cite{nips2014-best,neyshabur2015symmetric},   
which is a popular method for MIPS. We used the nearest neighbor transformation  proposed in \cite{bachrach2014speeding} and the LSH function, 
as suggested in \cite{neyshabur2015symmetric}. We used the standard amplification procedure, i.e., the final result is an OR-construction of $b$ hyper LSH hash     functions and each hyper LSH function is an AND-construction of $a$ random projections.
\item \GM{}~\cite{yu2017greedy}, which is a recently proposed method that uses a budget $B$ to control the time complexity of the query time. 
\item \PCA{}~\cite{bachrach2014speeding}, which uses the depth of the PCA tree to control the time-precision tradeoff.
\end{itemize}

\ph{Comparison Metrics}
We compare different algorithms by  varying their parameters in order to explore their tradeoffs between precision and online speedup.
Precision is defined as the fraction of true top $K$ solutions in the returned top $K$ solutions. 
Online speedup of an algorithm is defined as the query time required by the na\"ive  (i.e., exhaustive) search
	divided by the query time of that algorithm.
Recall that, unlike the baselines, our algorithm does not require any preprocessing. 
However, we ignore the preprocessing time of the baselines in our comparisons, showing that our algorithm's online speedup 	
	is still superior despite the lack of any preprocessing.

\subsection{Characteristics of the \ME{} Algorithm}
\label{sec:charac_me}
  Theorem \ref{THM:pac_me} provides a PAC bound for \ME{}. In other words, with the 
$\delta$ and $\epsilon$ provided by the user, \ME{} is guaranteed to return an $\epsilon$-optimal 
solution with probability at least $1-\delta$. Note that this is a worst-case guarantee, and in most cases
 	we expect the returned solution to be much better than that. 
Therefore, to empirically validate our worst-case guarantee, we design an adversarial dataset as follows (we will use other realistic datasets in 
later experiments).

To generate an adversarial dataset, 
	we use $10^4$ arm, each with a list of $10^5$ reward values.
For each arm $a$, we choose its true mean $r_a$ uniformly at random from $[0, 1]$.
Then, the rewards for that arm are generated with every reward being $1$ with probability $r_a$ and being $0$ with probability $1-r_a$.
When an arm is pulled---i.e., a sample is drawn from the rewards list without replacement---the rewards with value $1$ are returned 
	before those with value $0$. This is to make the arms as indistinguishable as possible to the algorithm, 
		thus causing an adversarial scenario. 

In this experiment, we vary $\epsilon$ between $0$ and $0.6$. For each value of $\epsilon$, we try all values of $\delta$ from 
the set $\{0.01, 0.05, 0.1, 0.2, 0.3\}$. For each pair of $\epsilon$ and $\delta$, we run \ME{} $20$ times, each time on a different randomly generated adversarial dataset (as described above). We then measure the $(1-\delta)$-percentile of the list of suboptimalities
	for each specific pair of $\epsilon$ and $\delta$. 
Figure~\ref{fig:syn_check} reports the average of these suboptimalities for each value of $\epsilon$. Since the suboptimality is always less than its corresponding value of $\epsilon$, it confirms that they are indeed
   smaller than  their corresponding values of $\epsilon$, i.e., validating Theorem \ref{THM:pac_me}.     

\begin{figure*}[!t]
   \begin{minipage}{0.31\textwidth}
     \vspace{0.1\linewidth}
     \centering
     \includegraphics[width=1.0\linewidth]{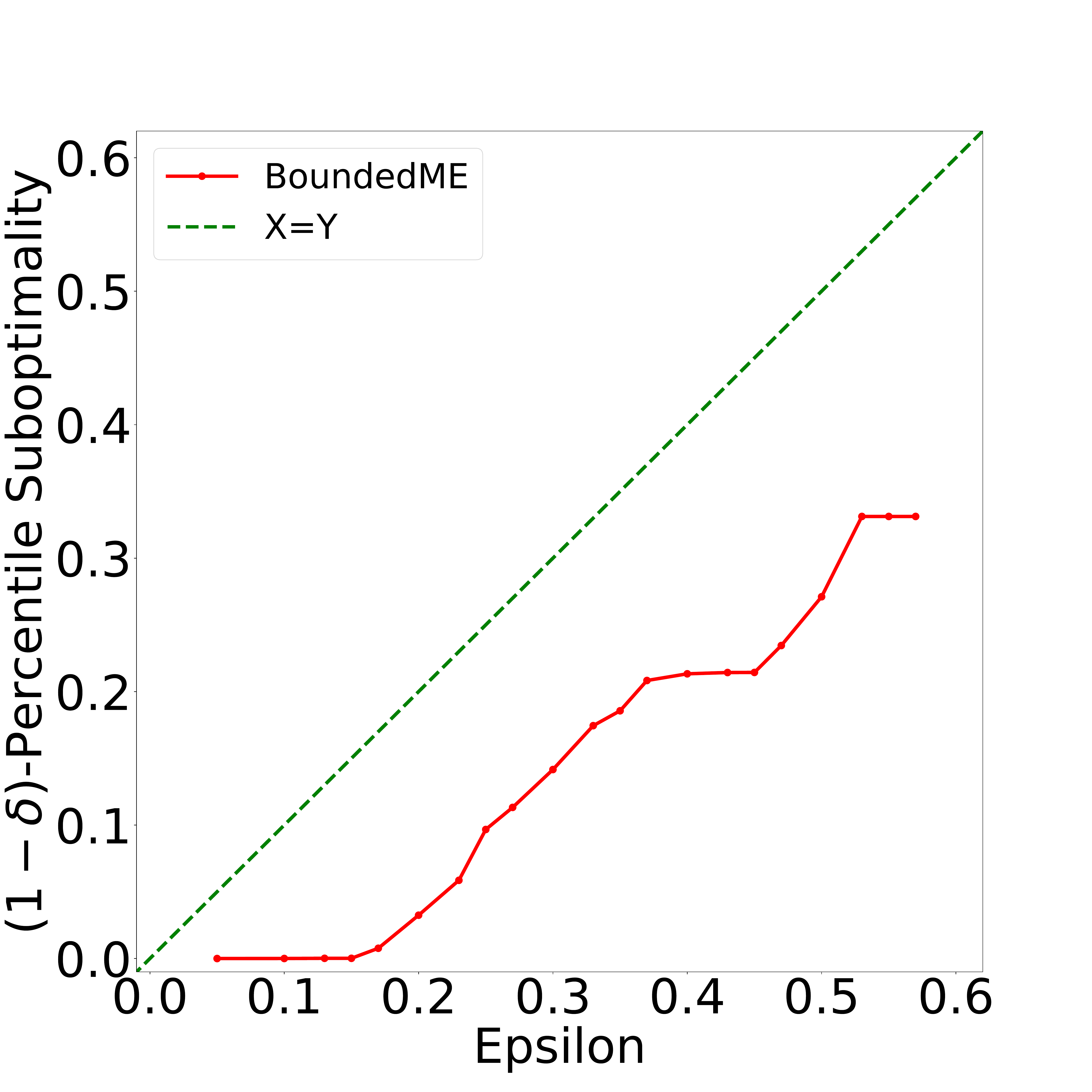}
     \vspace{0.3cm}
\caption{Correctness of \ME's guarantees}\label{fig:syn_check}
   \end{minipage}
   \begin{minipage}{0.8\textwidth}
     \centering
     \subfigure[top $5$ solutions]{
	 \includegraphics[width=.4\linewidth]{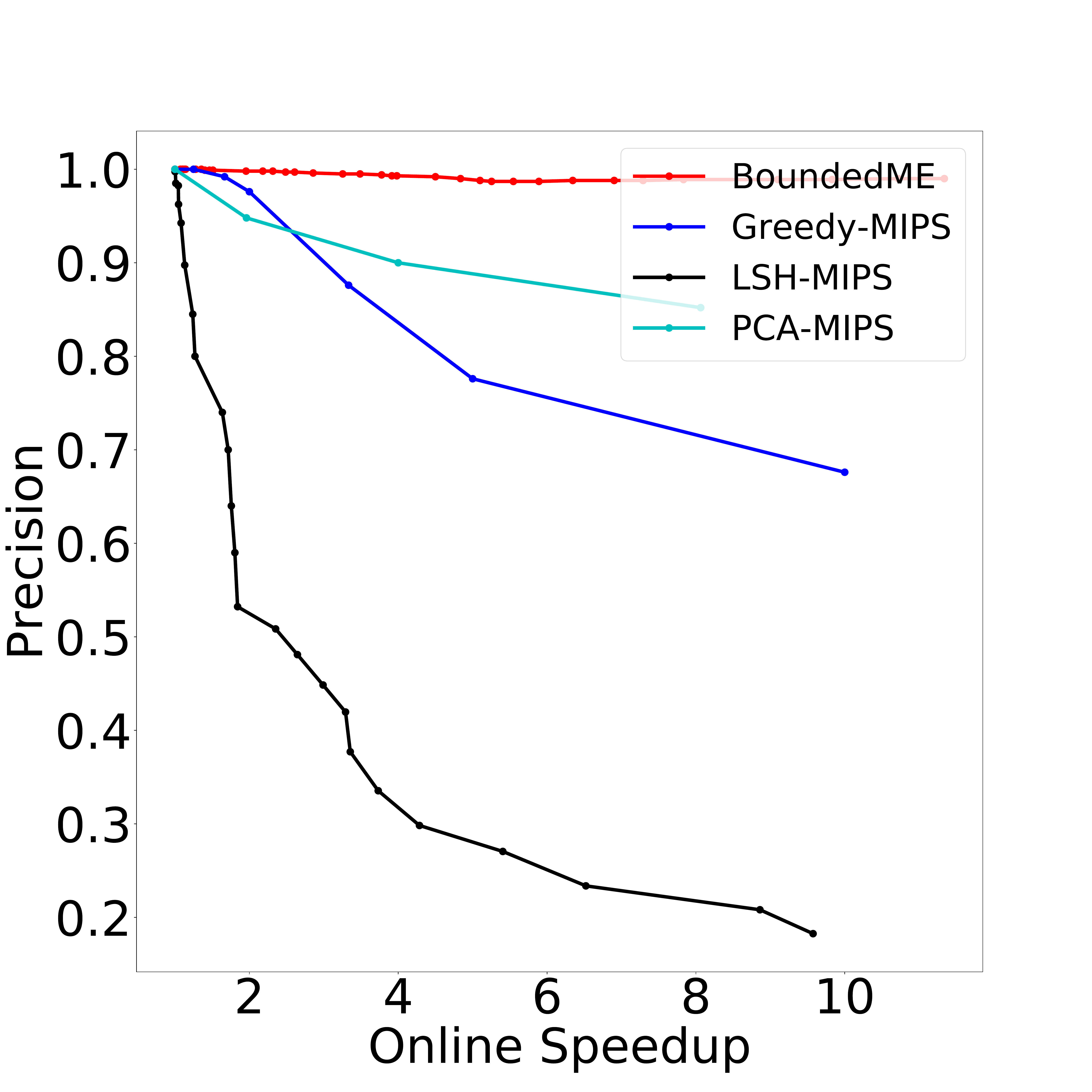}
     }
	 \subfigure[top $10$ solutions]{
     \includegraphics[width =.4\linewidth]{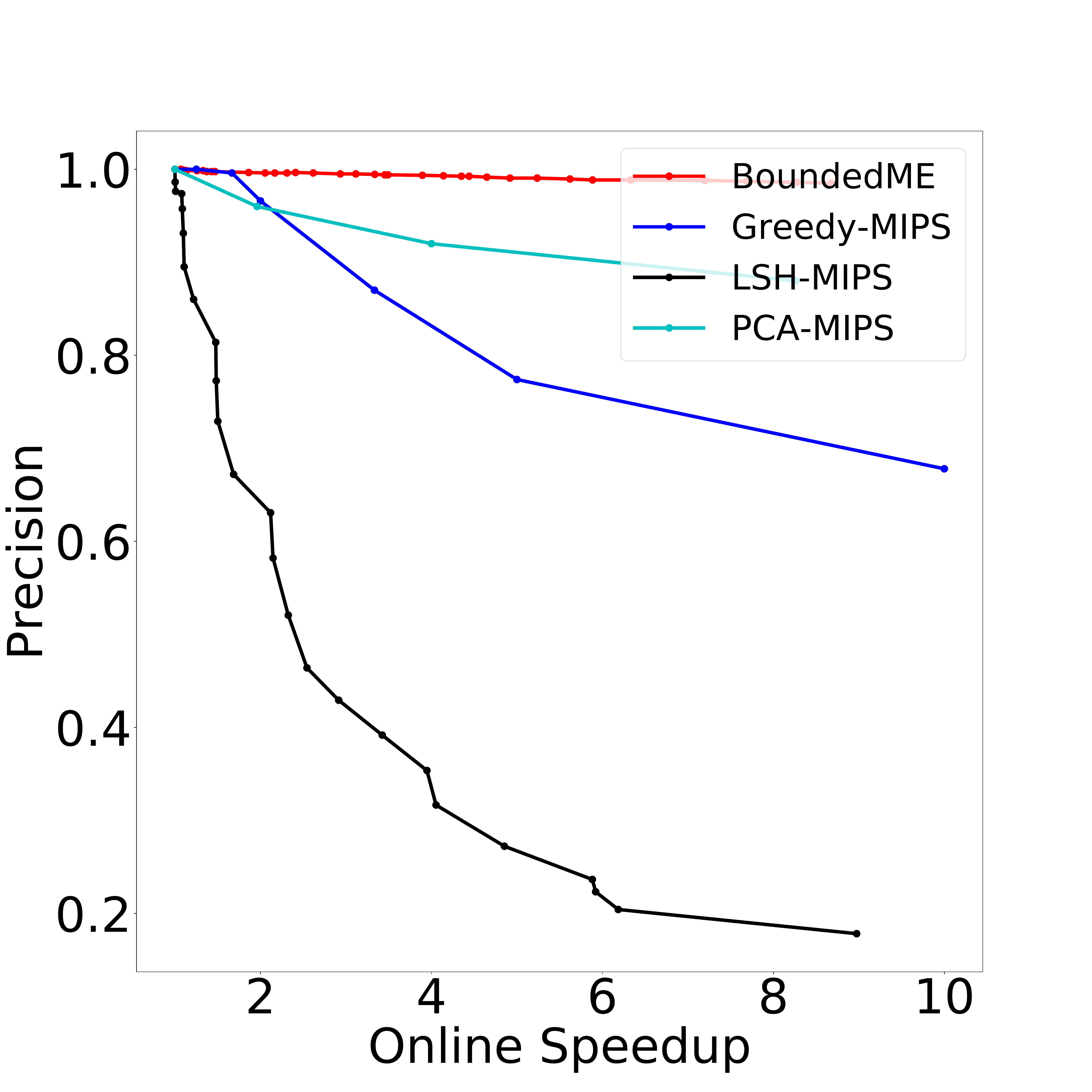}
	}
    \caption{Synthetic Gaussian dataset}\label{fig:syn_guassian}
   \end{minipage}
   \vspace{-0.2cm}
\end{figure*}

\subsection{\ME{} vs. Other MIPS Algorithms on Synthetic Datasets}
\label{expr:syn}

We generate two synthetic datasets, where the vector values are drawn from Gaussian  and uniform distributions, respectively.  
  For \ME{}, we varied  $\epsilon,\delta$$\in$$[0, 1]$. For \LSH{},  we varied $a$$\in$$[1,20]$ and $b$$\in$$[1, 50]$. For \GM{}, we varied $B$ from $10\%$ to $100\%$ of the dataset size. For \PCA{}, we varied the tree depth 
   in $[0, 20]$. We run experiments for both the cases of returning the top $5$ and $10$ solutions. 
As shown in Figures \ref{fig:syn_guassian} and \ref{fig:syn_uniform},  when the online speedup is small, all methods have very high precision. However, when online speedup becomes larger, the precisions achieved by other methods start to drop quickly, while \ME{} can still maintain high precision. This demonstrates that \ME{} outperforms these previous methods, despite its lack of preprocessing time.

\begin{figure}[!t]
\centering
\subfigure[top $5$]{
\includegraphics[width = 3.7 cm]{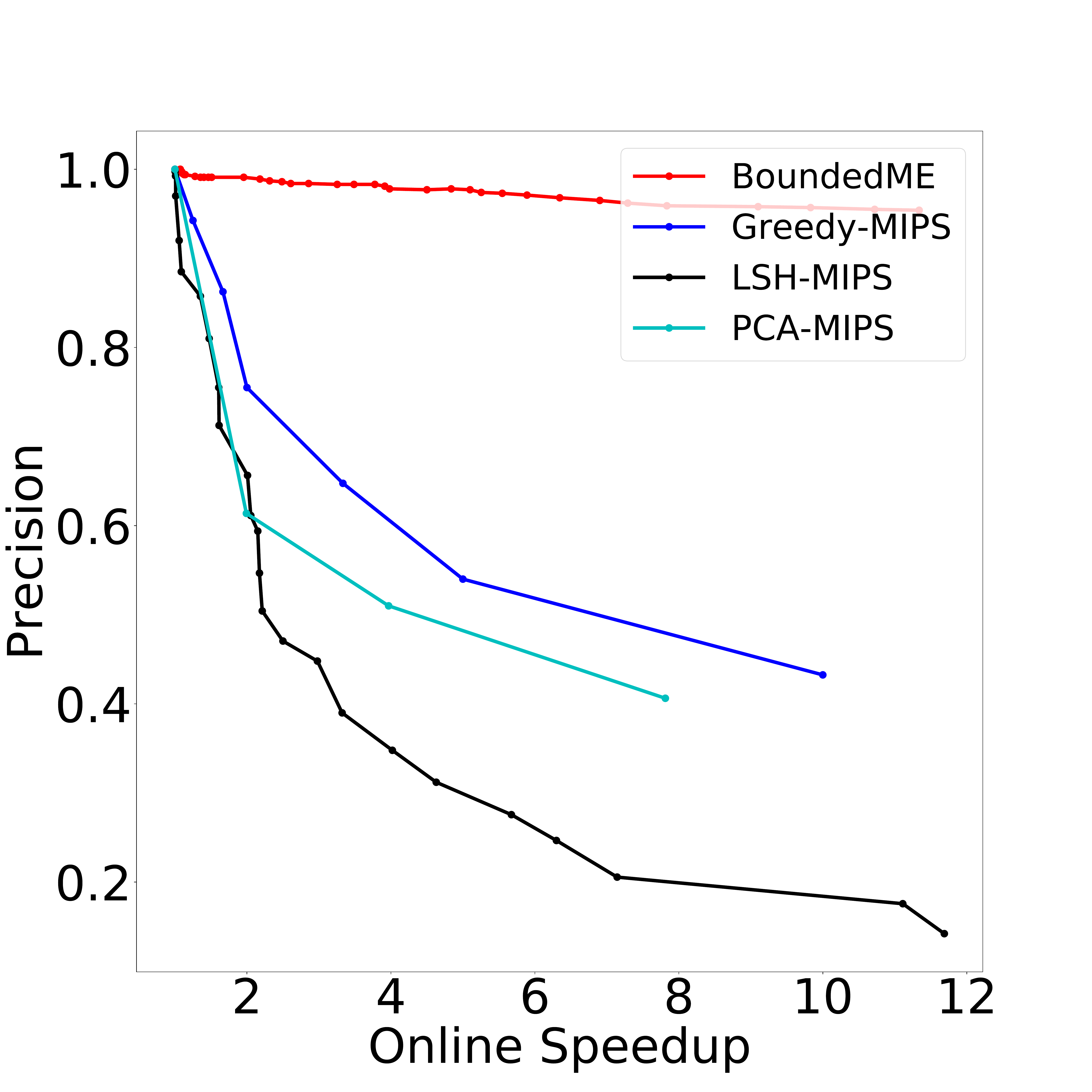}
}
\subfigure[top $10$]{
\includegraphics[width = 3.7 cm]{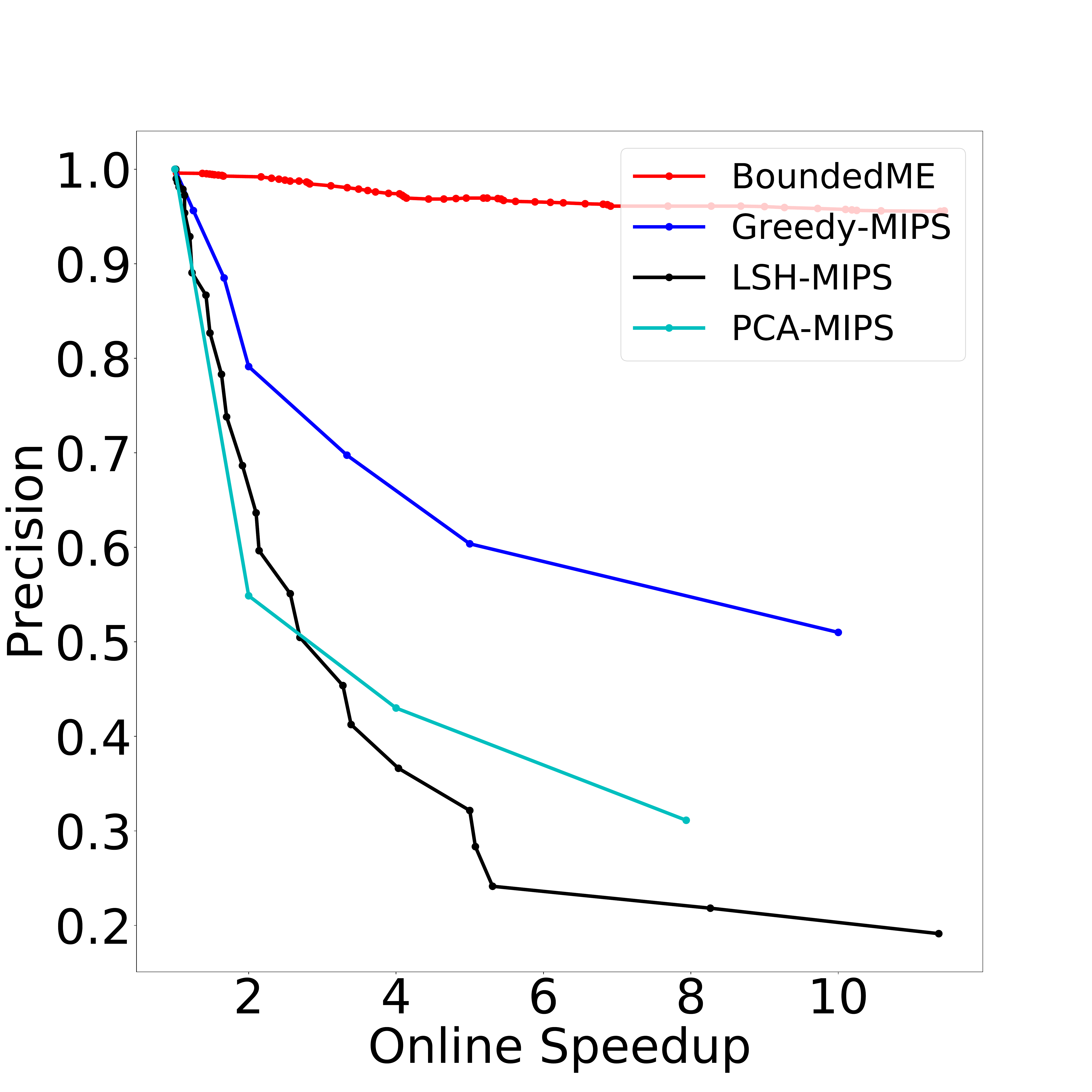}
}
\caption{Synthetic uniform dataset}
\label{fig:syn_uniform}
\end{figure}

\begin{figure}[!th]
\centering
\subfigure[Netflix Dataset]{
\includegraphics[width = 3.7 cm]{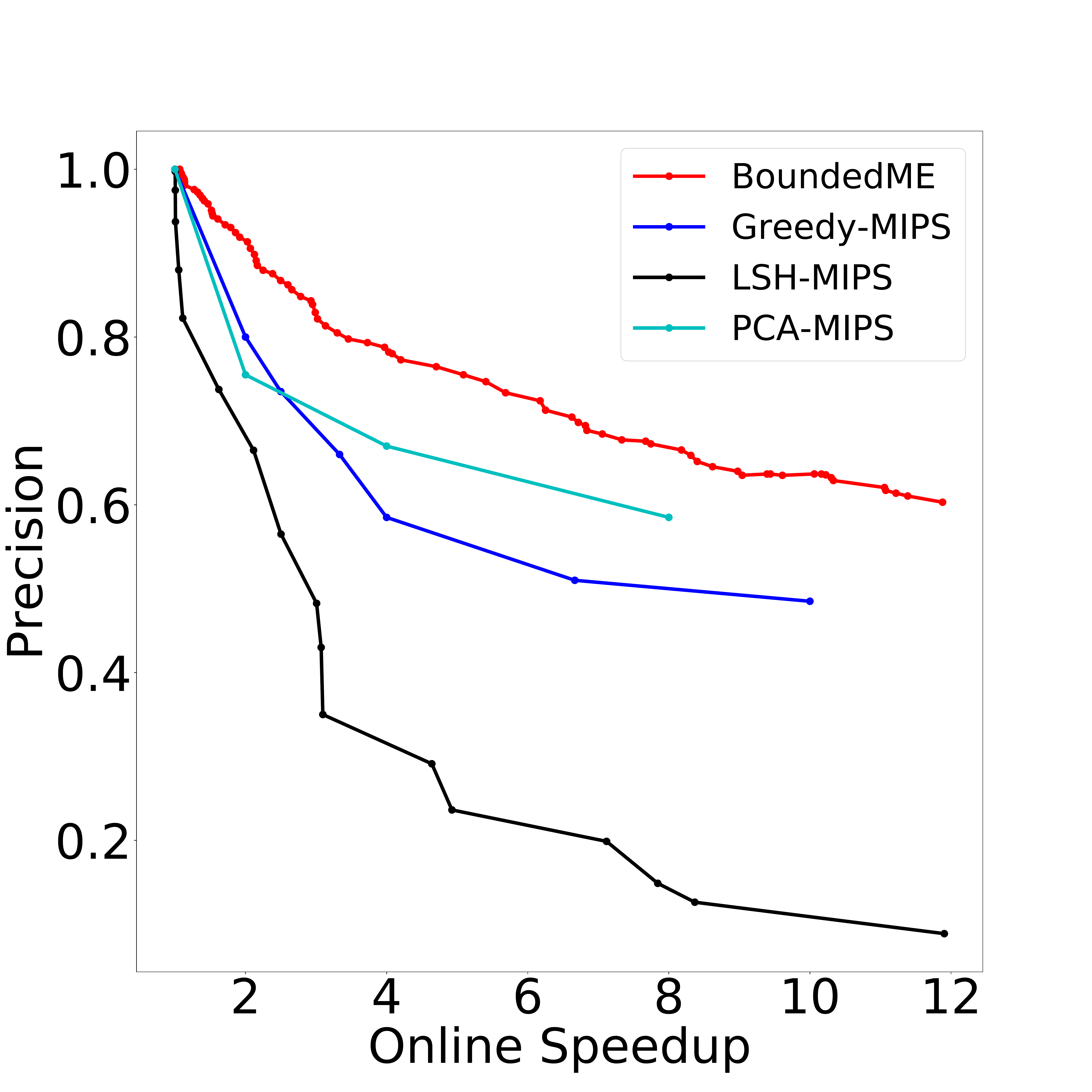}
}
\subfigure[Yahoo Dataset]{
\includegraphics[width = 3.7 cm]{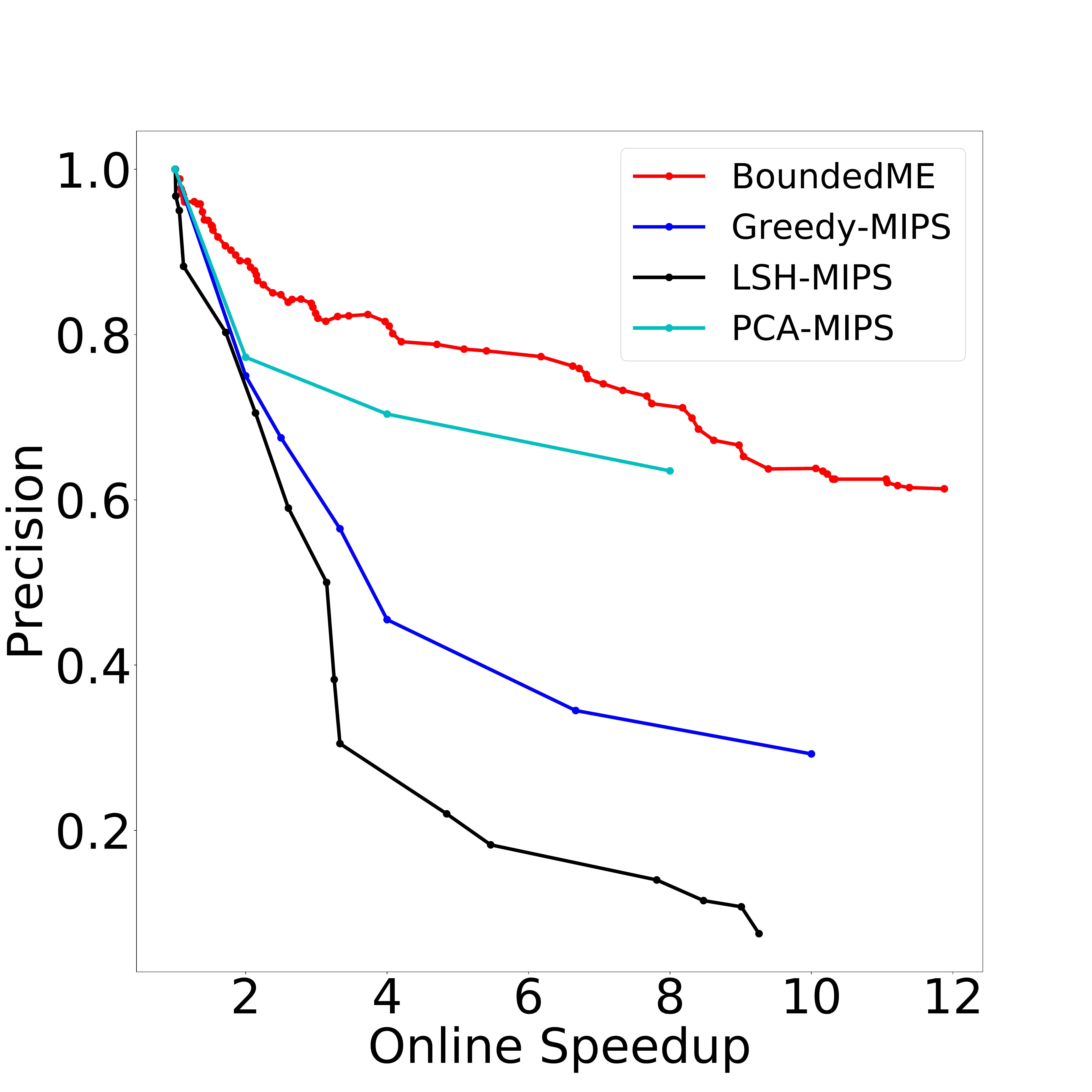}
}
\caption{Real-world datases}
\label{fig:real_world}
\end{figure}

\subsection{\ME{} vs. Other MIPS Algorithms  on Real-World Datasets}
 \label{expr:real}
 We also compare \ME{} against others on two real-world datasets, Netflix and Yahoo-Music  used   in \cite{yu2017greedy}. We use the same setting as in \cite{yu2017greedy} to compute the vector embeddings using matrix factorization. 
 The other parameters are the same as previous subsection for the case of identifying the top $5$ solutions.
Again,  as shown in Figure \ref{fig:real_world}, the precision of other methods drops more quickly than that of \ME{}, implying the superior performance of \ME{} over its counterparts.


\section{Conclusion}

We introduced a new bandit setting, \newBandit{} (MAB-BP), where the rewards are sampled \emph{without} replacement from a 
\emph{finite} list. We showed that this setting can be used for solving  important problems, such as 
Maximum Inner Product Search and Nearest Neighbor Search.
We also proposed a new algorithm, \ME{}, 
which extends the Median Elimination framework for MAB-BP settings.
Using a new concentration inequality for finite lists, we derived \ME's suboptimality guarantee and sample complexity. 
By applying \ME{} to MIPS, we improved on   state-of-the-art methods for MIPS by (1) avoiding their preprocessing step, 
	and (2) offering a knob to the user to directly control the suboptimaltiy of the results. 
We also conducted extensive experiments on both synthetic and real-world datasets,
	showing significant speedups over state-of-the-art MIPS algorithms.

\section{Acknowledgments}
This work is in part supported by National Science Foundation (grants 1629397 and 1553169). 
The authors would like to thank anynomous reviewers for their insightful comments.

\bibliographystyle{aaai}
\bibliography{mips,mozafari}
\clearpage
\section{A. Detailed Proofs}
Proofs that are omitted from the paper are included in this supplemental material.

We use the following corollary from \cite{bardenet2015concentration}. 

\begin{corollary}[Corollary 2.5 in \cite{bardenet2015concentration}] \label{CO:sample_without_replacement_original}
Let $\mathcal{X} = (x_1, x_2, \cdots, x_N)$ be a finite set of size $N>1$ in $[a, b]$ with mean $\mu = \frac{1}{N}\sum_{i=1}^Nx_i$, and $(X_1, \cdots, X_m)$ be a list of size $m< N$ sampled without replacement from $\mathcal{X}$.
  Then for any $m\leq N$, and any $\delta \in [0,1]$, it holds 
\begin{equation}
\mathbb{P}\left[ \frac{1}{m}\sum_{t=1}^m X_t - \mu \leq (b-a)\sqrt{\frac{\rho_m \log(1/\delta)}{2m}} \right] \geq 1-\delta
\end{equation}
where $\rho_m$ is defined as
\begin{equation}
\rho_m =\mathop{min} \left\{
(1-\frac{m-1}{N}),
(1-\frac{m}{N})(1+1/m)
\right\}
\end{equation}
\end{corollary}

Based on this corollary, we  have the following lemma.  
\begin{lemma} \label{LM:sample_without_replacement}
Let $\mathcal{X} = (x_1, x_2, \cdots, x_N)$ be a finite set of size $N>1$ in $[a, b]$ with mean $\mu = \frac{1}{N}\sum_{i=1}^Nx_i$, and $(X_1, \cdots, X_m)$ be a list of size $m< N$ sampled without replacement from $\mathcal{X}$.
Then, for any given  $0<\epsilon<1, 0<\delta<1$,  if 
\begin{equation}
m\geq \mathop{min}\{\frac{u+1}{1+\frac{u}{N}}, \frac{u+\frac{u}{N}}{1+\frac{u}{N}} \}
\end{equation}
where $u= \frac{\log(1/\delta)}{2} \frac{(b-a)^2}{\epsilon^2}  $,
then we have
\begin{equation}
\mathbb{P}\left[ \frac{1}{m}\sum_{t=1}^m X_t - \mu \leq  \epsilon  \right] \geq 1-\delta
\end{equation} 
\end{lemma}

\begin{proof}
From Corollary \ref{CO:sample_without_replacement_original}, it is obvious that for any given $\epsilon>0, \delta>0$, as long we have
\begin{equation} \label{EQ:m_rho_m}
\frac{m}{\rho_m} \geq \frac{\log(1/\delta)}{2} \frac{(b-a)^2}{\epsilon^2}
\end{equation}
then we also have
\begin{equation}
\mathbb{P}\left[ \frac{1}{m}\sum_{t=1}^m X_t - \mu \leq  \epsilon  \right] \geq 1-\delta
\end{equation}


Denote the right hand side of Eq.~\ref{EQ:m_rho_m} as $u$.
Since $\rho_m$ is the minimum of two functions, we can equivalently solve for $m$ by using either of the two function, and then take the 
	minimum of their solutions.

If $\rho= 1- \frac{m-1}{N}$ , we have
\begin{equation} 
\begin{split}
&\frac{m}{1-\frac{m-1}{N}} \geq u\\
&m \geq u - \frac{m-1}{N}u\\
&(1+\frac{u}{N})m \geq u+\frac{u}{N}\\
&m\geq \frac{u+\frac{u}{N}}{1+\frac{u}{N}}
\end{split}
\label{eq:case1}
\end{equation}

If $\rho = \frac{m}{(1-\frac{m}{N})(1+\frac{1}{m})}$ , we have
\begin{equation}
\begin{split}
&\frac{m}{(1-\frac{m}{N})(1+\frac{1}{m})} \geq u\\
&m \geq [(1+\frac{1}{m}) - \frac{m}{N} -\frac{1}{N}] u\\
&m\geq u+\frac{u}{m} - \frac{u}{N}m -\frac{u}{N}\\
&(1+\frac{u}{N})m - \frac{u}{m} \geq u-\frac{u}{N} \\
&(1+\frac{u}{N})m^2 -(u-\frac{u}{N})m - u \geq 0\\
\end{split}
\end{equation}

We need to find an $m$ that satisfies 
$(1+\frac{u}{N})m^2 -(u-\frac{u}{N})m - u \geq 0$. 
However, the closed-form solution would be quite complicated. 
We thus relax it slightly, as follows. 
We know that for any $m$ that satisfies $(1+\frac{u}{N})m^2 -(u-\frac{u}{N})m - u -1 \geq 0$, it must also satisfy $(1+\frac{u}{N})m^2 -(u-\frac{u}{N})m - u \geq 0$. 
Therefore, we relax this problem to $(1+\frac{u}{N})m^2 -(u-\frac{u}{N})m - u -1 \geq 0$. The solution of the relaxed problem is also a valid solution to the original problem. 
To solve the relaxed problem,
\begin{equation}
\begin{split}
(1+\frac{u}{N})m^2 -(u-\frac{u}{N})m - u -1 &\geq 0\\
[(1+\frac{u}{N})m - u -1][m+1] &\geq 0\\
m&\geq \frac{u+1}{1+\frac{u}{N}} \\
\end{split}
\label{eq:case2}
\end{equation}

Combining (\ref{eq:case1}) and (\ref{eq:case2}), 
we have
\begin{equation}
m\geq \mathop{min}\{\frac{u+1}{1+\frac{u}{N}}, \frac{u+\frac{u}{N}}{1+\frac{u}{N}} \}
\end{equation} 

We would like the sample size  to be as small as possible. Thus, we simply use the equal sign from above inequality.
	 It is easy to see that, for any $u\geq 0$, we have $0 \leq m \leq N $.
\end{proof}

\begin{lemma} \label{LM:one_round_pac}
For   Algorithm 1, at any iteration $l$,  we have
 \begin{equation}
\mathbb{P}\left[ \tilde{p}_{S_l}\leq \tilde{p}_{S_{l+1}} +\epsilon_l  \right] \geq 1-\delta_l
\end{equation}
\end{lemma}
\begin{proof}
Let us denote $z_l$ be the arm with K-th highest true mean of reward in $S_l$.
Let $\tilde{S}_{l}^{\epsilon} = \{ a\in S_l: p_{z_l} -p_a \leq \epsilon   \}$.
At the end of each round in Algorithm 1, we remove $\left \lceil\frac{|S_l| -K}{2} \right\rceil$ arms with the least empirical means. 
In other words, $K+\left \lfloor\frac{|S_l| -K}{2} \right\rfloor$ arms with the highest empirical means will remain in $S_{l+1}$.
When there are at least $K$ arms from $\tilde{S}_{l}^{\epsilon}$ that end up in $S_{l+1}$,
 the K-th highest true mean of the arms in $S_{l+1}$ will be within $\epsilon_l$ of the K-th highest true mean of the arms in $S_l$.
Hence, the event $\tilde{p}_{S_l}\leq \tilde{p}_{S_{l+1}} +\epsilon_l$ fails when there are
strictly more than $K+\left \lfloor\frac{|S_l| -K}{2} \right\rfloor - K = \left \lfloor\frac{|S_l| -K}{2} \right\rfloor$ arms in $S_l\setminus \tilde{S}_{l}^{\epsilon}$ whose empirical means are greater than the empirical mean of every arm in $\tilde{S}_{l}^{\epsilon}$. 

Let us compute the probability that an arm from $S_l\setminus \tilde{S}_{l}^{\epsilon}$ is better than every arm from $\tilde{S}_{l}^{\epsilon}$ in terms of the empirical mean.  
For any $j\in S_l\setminus \tilde{S}_{l}^{\epsilon}$, we have
\begin{equation}
\begin{split}
&\mathbb{P}\left[ \hat{p}_j \geq \hat{p}_i ~~ \forall i \in \tilde{S}_l^{\epsilon_l}   \right]\\
\leq & \mathbb{P}\left[ \hat{p}_j \geq \hat{p}_{z_l} \right]\\
\leq & \mathbb{P}\left[ \hat{p}_j \geq p_j + \epsilon_l/2 \text{ or }   \hat{p}_{z_l} \leq p_{z_l} - \epsilon_{l}/2   \right]\\
\leq & \mathbb{P}\left[ \hat{p}_j \geq p_j + \epsilon_l/2 \right] + \mathbb{P}\left[   \hat{p}_{z_l} \leq p_{z_l} - \epsilon_{l}/2   \right]
\end{split}
\end{equation}

According to Lemma \ref{LM:sample_without_replacement}, we know that after pulling each arm for 
$m\left( \frac{2}{\epsilon_l^2} \log\left( \frac{2(|S_l|-K)}{\delta_l \left(\left \lfloor\frac{|S_l| -K}{2} \right\rfloor +1  \right)} \right)    \right)$
times,  the above probability will be bounded by $\frac{\left \lfloor\frac{|S_l| -K}{2} \right\rfloor +1}{|S_l|-K} \delta_l$.
Let $n_{bad}$ be the number of arms in $S_l\setminus \tilde{S}_{l}^{\epsilon}$ that are better than every arm in $\tilde{S}_{l}^{\epsilon}$ in terms of their empirical mean of rewards.   Clearly,
\begin{equation}
\begin{split}
\mathbb{E}[n_{bad}] &\leq (|S_l|-K)\frac{\left \lfloor\frac{|S_l| -K}{2} \right\rfloor+1}{|S_l|-K} \delta_l\\
&\leq \left( \left\lfloor\frac{|S_l| -K}{2} \right\rfloor +1 \right)\delta_l 
\end{split}
\end{equation}

Using Markov inequality, we can bound the probability of failure as  
\begin{equation}
\begin{split}
&\mathbb{P}\left[ n_{bad} \geq \left\lfloor\frac{|S_l| -K}{2} \right\rfloor +1 \right]\\
\leq &\frac{\mathbb{E}[n_{bad}]}{\left\lfloor\frac{|S_l| -K}{2} \right\rfloor+1}\\
\leq & \delta_l
\end{split}
\end{equation}
\end{proof}

We can now prove the main property of Algorithm 1.   

\begin{theorem} \label{THM:pac_me}
The \ME{} algorithm (Algorithm 1) is guaranteed to return $\epsilon$-optimal solution with probability at least $1-\delta$.
\end{theorem}
\begin{proof}
It follows immediately from Lemma \ref{LM:one_round_pac} since 
$$\sum_{i =1}^{\log n} \delta_i =  \sum_{i =1}^{\log n} \frac{\delta}{2^i} \leq \sum_{i =1}^{\infty} \frac{\delta}{2^i} = \delta$$
 and 
 $$ \sum_{i =1}^{\log n} \epsilon_i  = \sum_{i =1}^{\log n} \frac{\epsilon}{4}\left(  \frac{3}{4}\right)^{i-1} \leq \sum_{i =1}^{\infty} \frac{\epsilon}{4}\left(  \frac{3}{4}\right)^{i-1} = \epsilon$$
\end{proof}

\begin{corollary}
For each arm, the number of times it is pulled by Algorithm 1 is upper-bounded by $N$.
\end{corollary}
\begin{proof}
Based on the definition of function $m(u)$, 
	it is not hard to see that $m(u)\leq N$ if $u\geq 0$.
Moreover, at round $l$ of Algorithm 1, the accumulated number of times that an arm in $S_l$ has been pulled
is
\begin{equation}
t_l = m\left( \frac{2}{\epsilon_l^2} \log\left( \frac{2(|S_l|-K)}{\delta_l \left(\left \lfloor\frac{|S_l| -K}{2} \right\rfloor +1  \right)} \right)    \right)
\end{equation}
Clearly, $t_l\leq N$. Thus, the total number of times any arm is pulled is upper-bounded by $N$.
\end{proof}

 \begin{corollary}
The time complexity of Algorithm 1 is $O\left( \frac{n\sqrt{N}}{\epsilon}\sqrt{\log\left(\frac{1}{\delta} \right)}  \right)$.
\end{corollary}
\begin{proof}
At round $l$, we know that $\epsilon_l = \frac{\epsilon}{4}\left(\frac{3}{4}\right)^{l-1}$ and $\delta_l = \frac{\delta}{2^l} $. 
The number of arms that will be removed at the end of  the $l$-th iteration is no greater than $\frac{n}{2^l}$.
It is easy to see that for the arms that are removed at the end of the $l$-th iteration, the total number of pulls for any of them will be $t_l$.
Further, we know that $m(u)\leq O(\sqrt{Nu})$ for $u>0$. Also, note that $K$ is typically a small constant, e.g., 1, 5, or 10. 
Therefore, the sample complexity is 
\begin{equation}
\begin{split}
&\sum_{l=1}^{\log n} n_l t_l \\
\leq & \sum_{l=1}^{\log n} \frac{n}{2^l} m\left( \frac{2}{\epsilon_l^2} \log\left( \frac{2(|S_l|-K)}{\delta_l \left(\left \lfloor\frac{|S_l| -K}{2} \right\rfloor +1  \right)} \right)    \right)\\
\leq & \sum_{l=1}^{\log n} \frac{n\sqrt{N}}{2^l} \frac{\sqrt{2}}{\epsilon_l} \sqrt{\log\left( \frac{2(|S_l|-K)}{\delta_l \left(\left \lfloor\frac{|S_l| -K}{2} \right\rfloor +1  \right)} \right)}\\
\leq & O\left( \frac{n\sqrt{N}}{\epsilon}\sqrt{\log\left(\frac{1}{\delta} \right)}  \right)
\end{split}
\end{equation}
\end{proof}

\end{document}